 \pdfoutput=1 
\documentclass[accepted]{uai2024} % after acceptance, for a revised version; 
       \usepackage{microtype}
\usepackage{graphicx}
\usepackage{subfigure}
\usepackage{booktabs} % for professional tables
\usepackage{multirow}

\usepackage{hyperref}
\usepackage{algpseudocode}
\usepackage[algo2e,ruled]{algorithm2e}
\algnewcommand{\LC}[1]{\\ \(\triangleright\) #1}
\usepackage{bbm}
\usepackage{derivative}
\usepackage{upgreek}
\usepackage{amssymb,amsmath, amsthm}
\usepackage{enumitem}
\usepackage{algorithm}

\theoremstyle{definition}

\usepackage{amsmath}
\usepackage{amssymb}
\usepackage{mathtools}
\usepackage{amsthm}
\newcommand{\eat}[1]{}
\usepackage[capitalize,noabbrev]{cleveref}

\theoremstyle{plain}
\newtheorem{theorem}{Theorem}[section]

\theoremstyle{definition}
\newtheorem{definition}{Definition}

\theoremstyle{remark}

\usepackage[textsize=tiny]{todonotes}

\newcommand{\X}{\mathbf{X}}
\newcommand{\x}{\mathbf{x}}

\newcommand{\F}{\mathbf{F}}
\newcommand{\p}{\mathbf{p}}

\newcommand{\pname}{Marginal Dominant}
     
\usepackage[american]{babel}
\usepackage{natbib} % has a nice set of citation styles and commands
\usepackage{mathtools} % amsmath with fixes and additions
\usepackage{booktabs} % commands to create good-looking tables
\usepackage{tikz} % nice language for creating drawings and diagrams
\makeatletter
\newcommand{\printfnsymbol}[1]{%
  \textsuperscript{\@fnsymbol{#1}}%
}
\makeatother
\title{Credibility-Aware  Multi-Modal Fusion Using Probabilistic Circuits}

\author[1]{Sahil Sidheekh\printfnsymbol{1}}
\author[1]{Pranuthi Tenali\printfnsymbol{1}}
\author[1]{Saurabh Mathur\thanks{Equal contribution}}
\author[2]{Erik Blasch}
\author[3]{Kristian Kersting}
\author[1]{Sriraam Natarajan}
\affil[1]{%
    Erik Jonsson School of Engineering \& Computer Science, The University of Texas at Dallas
}
\affil[2]{%
    Air Force Research Laboratory, Rome, NY, USA
}
 
\affil[3]{%
    Department of Computer Science , TU Darmstadt
} 
\begin{document}
\maketitle
\begin{abstract}
We consider the problem of late multi-modal fusion for discriminative learning. Motivated by noisy, multi-source domains that require understanding the reliability of each data source, we explore the notion of \emph{credibility} in the context of multi-modal fusion. We propose a combination function that uses probabilistic circuits (PCs) to combine predictive distributions over individual modalities. We also define a probabilistic measure to evaluate the credibility of each modality via inference queries over the PC. Our experimental evaluation demonstrates that our fusion method can reliably infer credibility while maintaining competitive performance with the state-of-the-art. 
\end{abstract}

\section{Introduction}
Real-world decision-making in high-profile tasks such as healthcare requires learning and reasoning reliably by utilizing the diverse modalities of available data sources. While such multi-modal data offer rich representations and potentially multiple views of the underlying phenomena (for example, images vs blood tests in a clinical setting), they also make learning and inference quite challenging. Raw data from different sources is often noisy, incomplete, and inconsistent. This heterogeneity poses a significant obstacle to effective data fusion and analysis.

Multi-modal fusion techniques~\citep{baltruvsaitis2018multimodal} have emerged as a promising approach to combine information from multiple sources to enhance performance on discriminative learning tasks. These techniques aim to extract and integrate complementary information from different modalities, leading to more robust and reliable results. However, a crucial aspect that often remains overlooked in multimodal fusion is the {\em explicit modeling of the credibility} of the information sources. In many applications, such as sensor fusion~\cite{khaleghi2013multisensor}, medical diagnosis~\cite{kline2022multimodal}, and financial analysis~\cite{sawhney2020multimodal}, the quality and reliability of the information sources vary significantly. Distinguishing reliable sources from non-reliable sources is essential for making accurate and informed decisions. Multimodal fusion methods often assume that all sources are equally credible, which can lead to suboptimal performance or even erroneous conclusions.

Credibility-aware methods in the context of late multimodal fusion have previously used weighted average~\cite{rogova2004reliability}, discounting factors~\cite{elouedi2004assessing} and Bayesian networks~\cite{wright2006credibility}. This results in models of credibility that are either too simple (as in the case of weighted averages and discounting factors) to model complex dependencies or too complex to perform tractable inference/reasoning (as in the case of general Bayesian networks or more recent deep models). 
We focus on {\bf multi-modal discriminative learning and propose a late fusion method that uses Probabilistic Circuits (PCs)}~\cite{Choi20}, to effectively combine the predictive distributions over individual modalities. PCs are a class of generative models that are expressive enough to model complex distributions while tractable for exact inference. Using the tractability of PCs, we define a probabilistic measure for assessing the credibility. Some salient features of our approach are that the use of PCs (1) allows for modeling uncertainty over unimodal predictive distributions effectively; (2) makes the model robust to noise and outliers; (3) enables effective handling of missing data; (4) is grounded in a robust theoretical framework; and (4) finally, makes it possible to obtain faithful estimates of credibility. 

Our paper makes the following key contributions: (1) To our knowledge, we introduce the first theoretically grounded multimodal fusion with strong probabilistic semantics based on PCs; Specifically, we identify the class of PCs that are amenable to this task of credibility-aware multi-modal fusion and define their characteristics; (2) We present two versions of our late fusion algorithm with different characteristics; (3) We derive a theoretically grounded measure of credibility and illustrate its connection to the conditional entropy over unimodal predictive distributions, allowing for reliable late fusion; (4) Finally, we experimentally validate the efficacy of PCs in modeling complex interactions between modalities and faithfully estimating their credibility. 

The rest of the paper is organized as follows: we begin with a concise overview of essential background and relevant work. Following this, we formulate the problem at hand and our PC-based fusion method, along with the architectural details and methodology for assessing credibility. We then experimentally evaluate the effectiveness of our method and finally conclude by summarizing our findings, contributions, and future work.

% (Rough Motivation) Need for a late fusion methodology that
% \begin{itemize}
%     \item has probabilistic semantics and can model the uncertainty over unimodal predictive distributions
%     \item is robust to noise and outliers
%     \item can handle missing data
%     \item can provide faithful estimates of credibility and is theoretically grounded 
% \end{itemize}

\section{Background}

% \textbf{Notation} To Add

\subsection{Multi-Modal Fusion}
Multi-modal fusion~\citep{baltruvsaitis2018multimodal} involves the integration of information from diverse sources or modalities. This field harnesses the potential of combining data of various types, such as text, images, and audio, to improve decision-making, pattern recognition, and predictive modeling. There are three broad approaches to multi-modal fusion in the discriminative learning setting, namely, early fusion, intermediate fusion, and late fusion. 

Early fusion approaches fuse information from multiple sources at the input level, typically ahead of feature extraction. A simple way to achieve this would be to combine raw modality features via concatenation or pooling via operations such as average, min, max, etc. \citep{baltruvsaitis2018multimodal}. In more complex deep learning models, early fusion is typically achieved by learning joint feature spaces~\citep{gadzicki2020early}. Apart from the curse of dimensionality, feature aggregation results in the loss of information about source-specific distributions~\citep{schulte2014aggregating}. This makes it difficult to infer the credibility of input sources.

In intermediate fusion, features extracted from each modality undergo further processing and transformation into a combined, higher-level representation~~\citep{Joze2019MMTMMT, Zhang2019CPMNetsCP, PrezRa2019MFASMF}. This approach offers more flexibility compared to early fusion, as the fusion process can take into account the characteristics of each modality individually. This can benefit learning representations, which can be used for fusion even when there's information missing from certain modalities ~\citep{Zhang2019CPMNetsCP}. However, inferring the credibility of individual input modality remains difficult due to the combined nature of representation used by the classifier. 

On the other hand, late fusion approaches combine the information from multiple sources by making predictions on each source independently and then combining the predictions. Combining rules~\citep{natarajan2005learning, manhaeve2018deepproblog} like weighted mean~\citep{shutova2016black} and Noisy-OR~\citep{tian2020uno} are commonly used for late fusion. 
% Weighed mean is defined as 
% $$ f(p1,\dots,p_m) = \sum_{i=1}^n w_i p_i,$$
% where $p1,\dots,p_m$ are the input probabilities and $w_i \in [0,1]$ are the weight parameters. Noisy-OR is defined as 
% $$ f(p1,\dots,p_m) = 1 - \prod_{i=1}^n (1-p_i).$$
While these combining rules allow explicit modeling of the importance of each source, they assume independence of the influence of each source on the target. Late fusion in deep learning models is implemented via additional feedforward layers ~\citep{glodek2011multiple, ramirez2011modeling}. This allows them to model complex correlations and influences of the sources on the target. However, this also makes it difficult to model the credibility of each source since neural network layers are opaque.

\subsection{Credibility}
Combining information from multiple, heterogeneous sources requires information fusion systems to account for the credibility of each modality's contribution~\citep{de2018uncertainty}. Credibility, as distinct from reliability, focuses on the information's truthfulness, while reliability relates to the source's consistency~\citep{blasch2013urref}. While human experts might estimate their information's credibility (self-confidence), automated sources require external evaluation~\citep{blasch2014urref}.

We follow prior works that approach the problem of accounting for source reliability in multimodal fusion from the perspective of the credibility of the information provided by the source. These works perform multimodal fusion using source-reliability coefficients learned using domain and contextual information ~\citep{Nimier98, Fabre2001PresentationAD}. In the absence of such information, an alternate approach involves learning these coefficients from data. This is achieved by minimizing the distance between a vector of beliefs resulting from fusion and a target vector from the training set ~\citep{rogova2001reinforcement,1262555}. Another data-driven method for establishing reliability is based on \textit{separability}, wherein the average statistical separability of information classes in each source is considered~\citep{Benediktsson1990}. This category of methods i.e. learning coefficients from training data, proves useful in establishing the relative credibility of the predictions of classifiers.

\subsection{Probabilistic Circuits (PC\MakeLowercase{s})}
Probabilistic circuits ~\citep{Choi20} are a class of generative models that represent the joint distribution over a set of random variables (say $\X$) using computational graphs that comprise three types of nodes - sum and product nodes as internal nodes, and simple tractable distributions at the leaves. Formally, a PC is defined as the tuple $(G = (V, E), \theta)$ where the Directed Acyclic Graph $G$ represents the computational graph structure and $\theta$ is the set of learnable parameters. The output of the root node gives the joint distribution modeled by the PC, which can be recursively obtained as:
\begin{align*}
\small
    P_n(\X = \x) = \begin{cases}
        \sum_{c \in \textbf{ch}(n)} w_cP_{c}(\X = \x)& n \in \text{Sum}\\
        \prod_{c \in \textbf{ch}(n)} P_{c}(\X^{\textbf{sc}(c)} = \x^{\textbf{sc}(c)})& n \in \text{ Product}\\
        \psi_{n}(\X = \x) & n\in\text{Leaf}\\
    \end{cases}
\end{align*}
where $\textbf{ch}(n)$ gives the children of node $n$, $\textbf{sc}(n)$ gives the scope of node $n$ and $\psi_n$ is the probability density (or mass) function associated with the leaf node $n$.

The key advantage of PCs is that they admit tractable and often linear time inference for a variety of probabilistic queries under mild assumptions about the structure of $G$. 
In this work, we consider a subclass of PCs that are \textit{smooth} and \textit{decomposable} (typically called sum-product networks ~\cite{SPNPoon2011}). 
% For a PC to be a valid SPN, it must satisfy two properties, namely, \textit{smoothness} and \textit{decomposability.}
A PC satisfies smoothness if the scope of each sum node is identical to the scope of each of its children. It satisfies decomposability if, for each product node, all the children have disjoint scopes. Smoothness and decomposability allow us to tractably infer marginal and conditional distributions from the learned joint.

The structure of PCs can be learned recursively via greedy heuristics ~\citep{gens2013learnspn, rooshenas14learning, dang20astrudel}, or by latent-space decomposition ~\citep{Adel2015LearningTS}. However, structure learning can be costly for large-scale data, and recent approaches rely on random and tensorized structures that resemble deep neural models ~\citep{Mauro17,peharz_20_einsum,peharz20a-rat-spn,sidheekh2023probabilistic} to achieve state-of-the-art performance.

\section{Multimodal fusion via PC\MakeLowercase{s}}

\begin{figure*}[!t]
    \centering
    \includegraphics[width=0.99\linewidth]{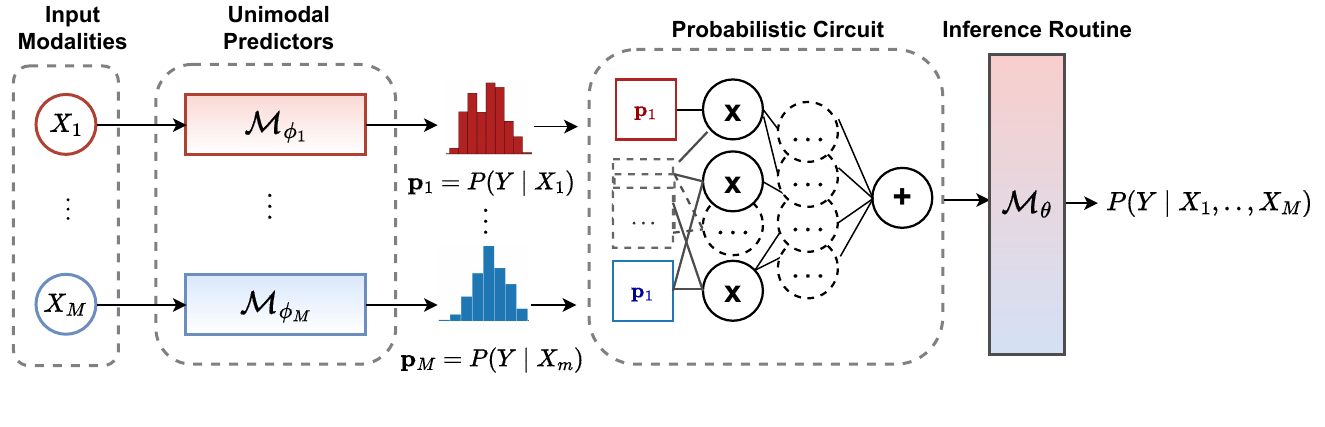}
    \caption{\textbf{Model Diagram} for our proposed PC-based fusion method. Each input modality $\X_i$ is processed by a unimodal predictor $\mathcal{M}_{\phi_i}$ to get the corresponding predictive distribution $\p_i$ over the target $Y$. A probabilistic circuit $\theta$ is used to model the joint distribution over the unimodal predictive distributions and $Y$, and the final prediction is obtained by running an inference routine over it, governed by the form of fusion function employed ($\mathcal{M}_{\theta}$).
    }
    \label{fig:architecture}
\end{figure*}

We begin by formalizing the {\em noisy late multi-modal fusion setting for discriminative learning} that we focus on. Given a dataset in which features predictive of a target concept are obtained from multiple different modalities, the late fusion setting involves training an expert over each modality to estimate the unimodal predictive distribution over the target and then combining them using a fusion function (probabilistic combination function in our case) to obtain the final output. More formally, 

%\fbox{\centering\begin{minipage}{.95\linewidth}
{\bf \textbf{Given:}} A dataset $\mathcal{D} = \{(\x_1^i,\x_2^i \ldots \x_M^i,y^i)\}_{i=1}^{N}$ with $N$ data points, each with information from $M$ different modalities,  i.e. each $\x^i_j \in \mathbb{R}^{d_j}$ where $d_j$ denotes the feature dimension corresponding to modality $j$ for the $i^{th}$ example, and $y^i$ denotes its target class. 

{\bf To do: } Learn a discriminative model $\mathcal{M}$ parameterized by $\{\theta, \phi = \{\phi_i\}_{i=1}^{m}\}$ that approximates the multimodal predictive distribution over $Y$\footnote{We use uppercase to denote random variables and lowercase to denote their corresponding values.}  as 
\begin{align*}
    P(Y|\X_1, \ldots, \X_M) &\approx \mathcal{M}_{\theta,\phi}(\X_1, \ldots, \X_M)\\
    &= \mathcal{M}_{\theta}(\mathcal{M}_{\phi_1}(\X_1),\ldots,\mathcal{M}_{\phi_M}(\X_M))
\end{align*}
where $\mathcal{M}_{\theta}$ is the fusion function, and $\mathcal{M}_{\phi_i}(\text{or } \mathcal{M}_i)$ is the unimodal predictor corresponding to modality $i$.
%}

In several applications, data inherently comes with a degree of noise which can affect the reliability of the information provided by each modality. Different modalities often offer complementary insights into the target  
$Y$; however in the presence of noise, they can potentially present conflicting information (an image might potentially present a conflicting finding to that of a blood test). This necessitates a fusion method that not only leverages the unique information within each modality to make accurate predictions but is also capable of evaluating the reliability of these predictions, providing a measure of each modality's credibility.

Thus, as a key contribution, we develop {\em a principled notion of credibility by taking a probabilistic view of the late multimodal fusion setting}. Let us denote by $\mathcal{F}_{\phi_j}$ the true predictive distribution over target $Y$ given modality $j$, i.e $\mathcal{F}_{\phi_j}=P(Y|\X_j)$. We consider the joint distribution over the unimodal predictors and the target $Y$  and define credibility as the relative amount of information contributed by  a modality to the multi-modal predictive distribution over the target $Y$, as follows:

\begin{definition}
    The \textbf{credibility} of a modality $j$ in predicting the target $Y$ is defined as the divergence between the conditional distributions over $Y$ given all unimodal predictive distributions $\{\mathcal{F}_{\phi_i}\}_{i=1}^{M}$ 
    % and the conditional probability distribution over $Y$ given all unimodal predictive distributions 
    including and excluding $\mathcal{F}_{\phi_j}$. i.e.
\begin{equation*}
    \mathcal{C}_j = \delta(P(Y \mid \{\mathcal{F}_{\phi_i}\}_{i=1}^{M})\ ||\ P(Y \mid \{\mathcal{F}_{\phi_i}\}_{i=1}^{M} \setminus \{\mathcal{F}_{\phi_j}\})) 
    \label{eq:credibility}
\end{equation*}
\end{definition}
where $\delta$ is a divergence measure, such as the KL-Divergence. It follows that $\mathcal{C}_j \geq 0 \ \forall j$, but can be unbounded. Thus, to facilitate easy comparison across modalities, we define the \textbf{relative credibility} score $\Tilde{\mathcal{C}}$ as 
$$ \Tilde{\mathcal{C}}_j = \frac{\mathcal{C}_j}{\sum_j \mathcal{C}_j}. $$
Note that $0 \leq  \Tilde{\mathcal{C}}_j \leq 1 \forall j$ and $\sum_j  \Tilde{\mathcal{C}}_j=1$, and is therefore a normalized and probabilistic measure for assessing the credibility of modality $j$.

We now outline more formally how the defined notion of credibility is related to the uncertainty over the unimodal predictive distribution. A well-established method for quantifying the uncertainty and information content within a random variable is through the concept of entropy. We present a theorem that correlates the credibility of a modality with the entropy of its predictive distribution, under mild assumptions, as defined below.

\begin{definition}
A model $\mathcal{M}$ representing a probability distribution ($P_\mathcal{M}$) over $n$ random variables $\X$ is said to be \textbf{\pname{}} if its marginals are lower bounded by the joint everywhere. \textit{i.e.,}
\begin{align*}
    P_\mathcal{M}(\X^{-j} = \x^{-j}) &\geq P_\mathcal{M}(\X^{-j} = \x^{-j}, \X^{j} = \x^{j})\ \  \\ &\forall \ (\x^j,\x^{-j}) \in Dom.(\X^j, \X^{-j})
\end{align*}
where $\ j \subseteq \{1,\dots,n\}$ and we use the notation $\X^{-j}$ to denote $\{X_i\}_{i=1}^{n} \setminus \{X_k\}_{k \in j}$ and $Dom.(\X)$ to denote the domain set of the variables in $\X$.
\end{definition}

\begin{theorem}
    The expected credibility $\mathcal{C}^{j}$ of a modality $j$ in predicting the target $Y,$ under a \pname{} distribution is lower bounded by the negative of the conditional entropy $(\mathbb{H})$ of the unimodal predictive distribution of modality $j$ over $Y,$ given the predictive distributions of all other modalities, i.e. 
    \begin{equation*}
    \mathbb{E}[\mathcal{C}^j] \geq - \mathbb{H}(\mathcal{F}_{\phi_j} | \{\mathcal{F}_{\phi_i}\}_{i=1}^{M} \setminus \{\mathcal{F}_{\phi_j}\})
    \end{equation*}
\end{theorem}
\begin{proof}
    Deferred to the appendix.
\end{proof}

Intuitively, a modality less corrupted by noise and more informative of the target than others can be expected to have a lower predictive entropy. Thus, by the above theorem, we can conclude that such a modality would always have a higher assigned credibility than others. Conversely, when a modality gets corrupted by noise, its credibility score decreases. Thus the defined measure of credibility is {\em theoretically grounded}. Its utility becomes evident in critical domains such as healthcare, where the stakes of decision-making are high. In such contexts, credibility assessments can guide the reliance on specific expert systems or enable the discounting of modalities that are deemed unreliable. 

\subsection{PC\MakeLowercase{s} as Combination Functions}

We now present the details of late fusion models $\mathcal{M}$ capable of incorporating the above-defined notion of credibility. It is clear that estimating credibility requires access to a generative model that estimates the joint distribution over $Y$ and the unimodal predictors $\{\mathcal{F}_{j}\}_{j=1}^{M}$. Additionally, the generative model should support efficient and exact evaluation of both joint and conditional probability densities.
Probabilistic Circuits (PCs) are one such class of generative models that can model complex distributions while supporting tractable and linear time inference of conditional and marginal distributions. Further, as we show below, the distribution modeled by a PC is \pname \ under certain structural properties, making it well-suited for credibility-aware fusion.

\begin{theorem}
    A Probabilistic Circuit is \pname \  if it is smooth, decomposable, and has leaf distributions with unimodal densities upper-bounded by unity.
\end{theorem}
\begin{proof}
    Deferred to the appendix.
\end{proof}

Thus, we define the fusion function using a PC $\theta$ that models the joint distribution over the unimodal predictors and the target $Y$. More formally, given unimodal experts $\{ \p_j = \mathcal{M}_{\phi_j}(\X_j)\}_{j=1}^{M}$ typically parameterized as deep neural networks, the PC models the distribution $P_\theta(Y, \mathbf{p}_1, \dots, \mathbf{p}_M)$. The PC can be viewed as a computational graph that recursively builds a complex joint distribution by taking sums and products of simpler distributions. We use categorical leaf distributions to model the target $Y$ and Dirichlet leaf distributions to model the unimodal predictive distributions  $\mathbf{p}_1,\dots,\mathbf{p}_M$. 

%The PC-based combining rule $f(p_1,\dots,p_m)$ is defined as $$P_\mathcal{M}(Y \mid p_1, \dots, p_m) = \frac{P_\mathcal{M}(Y, p_1, \dots, p_m)}{P_\mathcal{M}(p_1, \dots, p_m)}.$$

% \subsection{Inference}
The PC $\theta$ can be used to define the fusion function $\mathcal{M}_{\theta}$ in different ways. Since a PC supports exact conditional density evaluation, a straightforward way would be to define:
\begin{equation*}
    \mathcal{M}_{\theta}(\p_1, \p_2, \ldots, \p_M) = P_{\theta}(Y|\p_1, \p_2, \ldots, \p_M)
\end{equation*}

We will refer to this as the \textbf{Direct-PC (DPC)} combination function. It can explicitly model complex correlations between the influence of each source on the target while still being able to reason about their credibility. The resulting late fusion method allows both predictive inference and credibility assessment as elaborated below.

\textbf{Predictive Inference.} 
Given a multi-modal example, $(\x_1, \dots, \x_M),$ we can perform predictive inference over target $Y$ as follows:\\

(1) compute $\mathbf{p}_j = \mathcal{M}_{\phi_j}(\x_j) $ for each modality $j$ by evaluating the unimodal predictors $\mathcal{M}_i, \dots, \mathcal{M}_m$.
(2) Infer the multimodal predictive distribution over $Y$ given the unimodal distributions $\mathbf{p}_1, \dots, \mathbf{p}_M$ by performing conditional inference: 
\begin{align*}
    P_\theta(Y \mid \mathbf{p}_1, \dots, \mathbf{p}_M) &= \frac{P_{\theta}(Y, \mathbf{p}_1, \dots, \mathbf{p}_M)}{P_{\theta}(\mathbf{p}_1, \dots, \mathbf{p}_M)} \\
    &= \frac{P_{\theta}(Y, \mathbf{p}_1, \dots, \mathbf{p}_M)}{\sum_{y}P_{\theta}(Y=y,\mathbf{p}_1, \dots, \mathbf{p}_M)} 
\end{align*}
% \begin{enumerate}[wide, labelwidth=!, labelindent=0pt]
%     \item Compute $\mathbf{p}_j = \mathcal{M}_{\phi_j}(\x_j) $ for each modality $j$ by evaluating the unimodal predictors.
%     % $\mathcal{M}_i, \dots, \mathcal{M}_m.$
%     \item Infer the multimodal predictive distribution over $Y$ given the unimodal distributions $\mathbf{p}_1, \dots, \mathbf{p}_M$ by performing conditional inference: 
%     $P_\theta(Y \mid \mathbf{p}_1, \dots, \mathbf{p}_M) = \frac{P_{\theta}(Y, \mathbf{p}_1, \dots, \mathbf{p}_M)}{P_{\theta}(\mathbf{p}_1, \dots, \mathbf{p}_M)}$
% \end{enumerate}
\textbf{Credibility Assessment.} 
The credibility of a modality $j$ can then be estimated using the PC $\theta$ as 
\begin{equation*}
    \mathcal{C}^{\theta}_j = \delta(P_{\theta}(Y|\p_1, \ldots \p_M) ||P_{\theta}(Y|\p_1,  \ldots \p_{j-1}, \p_{j+1} \ldots \p_M))    
\end{equation*}
As smooth and decomposable PCs support linear time evaluation of joint, marginal, and conditional distributions, both predictive and credibility inference can thus be achieved in linear time. 

An alternative to the Direct-PC combination function, which explicitly utilizes the credibility scores would be to define the final predictive distribution as a convex sum of credibility-weighted unimodal predictive distributions. i.e: 
\begin{center}
    \vspace{-1em}
    $\mathcal{M}_{\theta}(\p_1, \ldots, \p_M) = \sum_{j=1}^{M} \left( \dfrac{\mathcal{C}^{\theta}_{j} }{\sum_{i=1}^{M}\mathcal{C}^{\theta}_{i}} \right) \p_j$
\end{center}
% \begin{equation*}
%     \mathcal{M}_{\theta}(\p_1, \ldots, \p_M) = \sum_{j=1}^{M} \left( \dfrac{\mathcal{C}^{\theta}_{j} }{\sum_{i=1}^{M}\mathcal{C}^{\theta}_{i}} \right) \p_j
% \end{equation*}
We refer to this combination function as the \textbf{Credibility-Weighted Mean (CWM)}. This approach allows us to weigh the predictive distributions according to the trustworthiness of the source, and is useful in ensuring that the final prediction reflects the most reliable and pertinent information available.
Figure \ref{fig:architecture} illustrates the overall architecture of our credibility-aware late-fusion approach.
% Here, $\mathcal{M}_1, \dots, \mathcal{M}_m$ are probabilistic unimodal discriminative models corresponding to each of the $m$ modalitities. Each model $\mathcal{M}_j$ induces a distribution over the target $Y$ conditioned on modality $j$ and can be implemented by any differentiable probabilistic classifier such as a Multilayer perceptron (MLP). Let this distribution be $\mathbf{p}_j = P(Y  \mid \X_j = \x_j).$ Late-fusion methods combine information from multiple modalities by defining a combining function over these probability distributions as the function $\mathcal{M}_\theta(\mathbf{p}_1, \dots, \mathbf{p}_m).$ 

Since PCs are differentiable computational graphs, they can be easily integrated with neural unimodal predictors and learned in an end-to-end manner using backpropagation and gradient descent.
We optimize the unimodal predictors to minimize the classification loss over both the unimodal predictions as well as the joint multimodal prediction. Further, we optimize the PC parameters to maximize the joint likelihood $P_{\theta}(Y,\p_1,\ldots,\p_M)$ as well as the  classification loss over the joint multimodal prediction.
Algorithm \ref{alg:credibility_weighted_mean} summarizes the overall training methodology for our proposed credibility-aware late multimodal fusion using PCs.

The adoption of PCs in our approach is primarily {\bf motivated by their tractability for probabilistic inference, which is instrumental in computing the probabilistic measures essential for assessing the credibility of each modality}. This tractability contrasts with the capabilities of more complex combination functions, such as neural networks, which, despite their potential for higher expressiveness and the ability to learn more intricate functions, do not inherently support the derivation of credibility measures.
PCs on the other hand offer a balance between expressiveness and tractability. Moreover, through the process of marginalization, PCs can naturally accommodate and adjust to the absence of data from one or more modalities, preserving the integrity of the inference process without requiring imputation or other preprocessing steps. This also enhances the robustness of the fusion method, ensuring reliable performance even when faced with incomplete data.

\begin{table*}[!h]
\centering
\begin{tabular}{@{}llllll@{}}
\toprule
\multicolumn{1}{c}{\textbf{Fusion Model}}     & \multicolumn{1}{c}{Accuracy} & \multicolumn{1}{c}{Precision} & \multicolumn{1}{c}{Recall} & \multicolumn{1}{c}{F1Score} & \multicolumn{1}{c}{AUROC} \\ \midrule
MLP                                                        & $\mathbf{72.43 \pm 0.15}$    & $\mathbf{72.20 \pm 0.31}$     & $\mathbf{71.97 \pm 0.18}$  & $\mathbf{71.93 \pm 0.23}$   & $96.29 \pm 0.07$          \\
Weighted Mean                                              & $66.00 \pm 1.03$             & $ 65.45 \pm 1.28$             & $65.48  \pm 1.12$          & $65.23 \pm 0.98$            & $95.25 \pm 0.05$          \\
Noisy-OR                                                   & $68.62 \pm 0.17$             & $ 68.06 \pm 0.46$             & $68.08 \pm  0.18$          & $67.76 \pm 0.21$            & $94.50 \pm 0.16$          \\
TMC                                                        & $69.95 \pm 0.11$             & $ 69.70 \pm 0.21$             & $69.45 \pm  0.15$          & $69.18 \pm 0.14$            & $94.99 \pm 0.11$          \\
\midrule
Credibility-Weighted Mean (Ours)                                & $70.41 \pm 0.15$             & $ 70.32 \pm 0.31$             & $69.46 \pm  0.27$          & $68.09 \pm 0.21$            & $94.82 \pm 0.16$          \\
Direct-PC (Ours)                               & $72.18 \pm 0.43$             & $ 71.70 \pm 0.35$             & $71.76 \pm 0.40$           & $71.63 \pm 0.36$            & $\mathbf{96.48 \pm 0.07}$      \\    
 \bottomrule
\end{tabular}
\caption{Mean test performance of late fusion methods on the \textbf{AV-MNIST} dataset, $\pm$ standard deviation across $3$ trials.}\label{tab:results_avmnist}
\end{table*}

\begin{table*}[!h]
\centering
\begin{tabular}{@{}llllll@{}}
\toprule
\multicolumn{1}{c}{\textbf{Fusion Model}}     & \multicolumn{1}{c}{Accuracy} & \multicolumn{1}{c}{Precision} & \multicolumn{1}{c}{Recall} & \multicolumn{1}{c}{F1Score} & \multicolumn{1}{c}{AUROC} \\ \midrule
MLP                                                        & $89.66 \pm 1.39$             & $90.38 \pm 1.32$              & $89.66 \pm 1.39$           & $89.56 \pm 1.38$            & $\mathbf{99.47 \pm 0.27}$          \\
Weighted Mean                                              & $91.33 \pm 2.25$             & $ 91.97 \pm 1.73$             & $91.33  \pm 2.25$          & $91.38 \pm 2.12$            & $99.39 \pm 0.33$          \\
Noisy-OR                                                   & $90.83 \pm 2.63$             & $ 91.39 \pm 2.39$             & $90.83 \pm  2.63$          & $90.86 \pm 2.56$            & $99.41 \pm 0.28$          \\
TMC                                                        & $91.50 \pm 3.24$              & $92.14 \pm 3.03$             & $91.50 \pm 3.24$           & $91.47 \pm 3.12$            & $99.45 \pm 0.29$          \\ 
\midrule
Credibility-Weighted Mean (Ours)                                & $\mathbf{92.49 \pm 1.41}$    & $ \mathbf{94.03 \pm 1.57}$    & $\mathbf{92.50 \pm 1.42}$  & $\mathbf{92.49 \pm 1.02}$   & $99.42 \pm 0.29$          \\ 
Direct-PC (Ours)                               & $91.67 \pm 1.02$             & $ 92.42 \pm 1.15$             & $91.67 \pm 1.02$           & $91.58 \pm 0.94$            & $99.28 \pm 0.40$    \\      
 \bottomrule
\end{tabular}
\caption{Mean test performance of late fusion methods on the \textbf{CUB} dataset, $\pm$ standard deviation across $3$ trials.}\label{tab:results_cub}
\end{table*}

\begin{algorithm2e}[t]
  \SetAlgoLined
  \algnewcommand{\LeftComment}[1]{\Statex \(\triangleright\) #1}
  \caption{Credibility Aware Late Fusion - Learning}
  \label{alg:credibility_weighted_mean}
  \SetKwInOut{Input}{input}\SetKwInOut{Output}{output}
  \Input{
  Multimodal Dataset $\mathcal{D}=\{ (\x_j^i,y^i)_{j=1}^{M}\}_{i=1}^{N}$, \\
  Unimodal Predictors $\{ \mathcal{M}_{\phi_{i}}\}_{i=1}^{M}$\\
  Probabilistic Circuit $\theta$,\\
  Loss function $l$, Divergence Measure $\delta$\\
  Learning rates $\eta_1, \eta_2$, \#Iterations $t_{max}$
  } 
  \Output{Optimal parameters:  $\Tilde{\theta}$, $\{ \Tilde{\phi_j} \}_{j=1}^{M}$}
  \textbf{initialize:} $\Tilde{\theta} = \theta, \{ \Tilde{\phi_j} = \phi_j \}_{j=1}^{M}, t = 1$ 
  
  \While{$t \leq t_\text{max}$} { 
    
    $\{ (\x_j^{i},y^{i})_{j=1}^{M} \}_{i=1}^{B} \sim \mathcal{D}$ \algorithmiccomment{Sample a mini-batch} 
    
    For each modality $j$ and data point $i$ 
    \LC{Compute unimodal predictive distributions ${\p^{i}_j}$}
    
    ${\p_j^i} \leftarrow \mathcal{M}_{\Tilde{\phi_j}}(\x_j^{i})$ 
    \LC{Obtain credibility scores}

    $\mathcal{C}_{j}^{i} \leftarrow \delta(P_{\Tilde{\theta}}(Y|\{\p^{i}_k\}_{k=1}^{M})||P_{\Tilde{\theta}}(Y|\{\p^{i}_k\}_{k=1}^{M}\setminus \p_j^i))$

    $\Tilde{\mathcal{C}}_{j}^{i} \leftarrow 
 \mathcal{C}_{j}^{i}/(\sum_{j=1}^{M}\mathcal{C}_{j}^{i})$ \LC{Compute the final predictive distribution}
    
    $\p^{i} \leftarrow  \sum_{j=1}^{M} \Tilde{\mathcal{C}}_{j}^{i}\p_{j}^{i}$ if CWM else  $P_{\Tilde{\theta}}(Y|\{\p^{i}_k\}_{k=1}^{M})$ 
    \LC{Compute the empirical loss}
    
    $L_{j} \leftarrow  \frac{1}{B}\sum_{i=1}^{B}l(\p_j^{i},y^i)$

    $L \leftarrow  \frac{1}{B}\sum_{i=1}^{B}l(\p^{i},y^i) + \sum_{j=1}^{M}L_j$ 
    \LC{Update the unimodal predictors and PC}
    
    $\{\Tilde{\phi_j}\}_{j=1}^{M} \leftarrow \{\Tilde{\phi_j}\}_{j=1}^{M} - \eta_1 \nabla_{\{\Tilde{\phi_j}\}_{j=1}^{M}} L $

    $ \Tilde{\theta} \leftarrow \Tilde{\theta} - \eta_2 \nabla_{\Tilde{\theta}} L + \eta_2 \nabla_{\Tilde{\theta}} 
 \sum_{i=1}^{B} P_{\Tilde{\theta}}(y^i,\{\p_j^i\}_{j=1}^{M}) $
    
    $t = t + 1$ 
  }  
  \Return{$\Tilde{\theta}, \{\Tilde{\phi_j}\}_{j=1}^{M}$}
         
\end{algorithm2e}

\section{Empirical Evaluation}

% \paragraph{Datasets.} For the experimentation, we employed two benchmark datasets designed for multimodal fusion, namely, AV-MNIST and MM-IMDb. The AV-MNIST dataset comprises two modalities: images depicting digits from 0 to 9 and their corresponding audio representations. Here, the task is to determine the digit based on the multi-modal input. The second data set, MM-IMDb, consists of movie poster images and textual data describing the movie plots. Given that each movie can be associated with multiple genres, the task here involves multilabel classification, in contrast to the multiclass classification task posed by AV-MNIST.
To experimentally validate the utility of the proposed approach, we conducted experiments on {\bf four different} multimodal datasets: Caltech UCSD Birds (CUB), NYU Depth (NYUD), SUN RGB-D, and AV-MNIST, focusing on the task of multi-class classification. Overall, we designed experiments to answer the following research questions:
\begin{enumerate}[leftmargin=2.5em]
    % \item[\textbf{(Q1)}] How does the PC-based combining rule compare with existing combining rules?
    \item[\textbf{(Q1)}] Can a PC-based combining rule efficiently capture intricate dependencies between modalities to achieve performance at par with existing methods?
    \item[\textbf{(Q2)}] Can the tractability of PCs be used to reliably infer credibility scores for each source modality?
    \item[\textbf{(Q3)}] Is the proposed credibility-aware fusion robust to noise?
\end{enumerate}

% \subsection{Datasets and Baselines}
\textbf{Baselines} We implemented $4$ baseline fusion functions as elaborated below for comparison:

1. \textbf{Weighted Mean} combination function that defines the multimodal predictive distribution as:
$P(Y|\X_1, \X_2, \ldots, \X_M) = \sum_{i=1}^{M} w_i P(Y|\X_i)
$ where $w_i$ are learnable weights such that $0 \leq w_i \leq 1$ and $\sum_{i=1}^{m}w_i=1$. The constraints on the weights ensure that the combination function outputs a valid distribution.

2. \textbf{Noisy-Or} combination function that defines the multimodal predictive distribution as: \\ 
$P(Y|\X_1, \X_2, \ldots, \X_M) = 1 - \prod_{i=1}^M (1-P(Y|\X_i))$

3. \textbf{Multi Layer Perceptron (MLP)} combination function that maps the vector of unimodal predictions $[P(Y|\X_i)]_{i=1}^{M}$ to the multimodal predictive distribution $P(Y|\X_1, \X_2, \ldots, \X_M)$ using a feedforward neural network having $2$ hidden layers with $64$ neurons.

4. \textbf{Dempster's} combination function, used in \textbf{TMC} (~\cite{han2021trusted}) allows evidence from different sources to be combined by fusing \textit{belief masses} and \textit{uncertainty masses}. This rule ensures that the confidence of the final prediction is high when the input modalities are less uncertain and low when the input modalities are highly uncertain. When faced with different modalities that has conflicting beliefs, this combination rule only fuses the shared parts, making the final prediction dependent only on the confident modalities when some of the modalities are more uncertain.

% 4. \textbf{Dempster's} combination function, used in \textbf{TMC} (\cite{han2021trusted}), deals with probability masses and belief functions. For two independent sets of masses $\mathcal{M}^1 = \{{\{b^1_k}\}_{k=1}^K, u^1\}$ and $\mathcal{M}^2 = \{{\{b^2_k}\}_{k=1}^K, u^2\}$, the joint mass $\mathcal{M} = \{{\{b_k}\}_{k=1}^K, u\}$ is given by: 
% $
%     \mathcal{M} = \mathcal{M}^1 \oplus \mathcal{M}^2 
    
%     b_k = \frac{1}{1 - C} ({b^1_k}{b^2_k} + {b^1_k}u^2 + {b^2_k}u^1), u = \frac{1}{1 - C}{u^1}{u^2}
% $
% where C = $\sum_{i\neq j}{b^1_i b^2_j}$ is a measure of the amount of conflict between the modalities, $b^i_k$ is the belief mass of the $k^{th}$ class of the $i^{th}$ modality and $u^i$ is the uncertainty of the $i^{th}$ modality.
% For any arbitrary m number of modalities, joint mass $\mathcal{M}$ can be calculated as \\
% $
% \mathcal{M} = \mathcal{M}^1 \oplus \mathcal{M}^2 \oplus ... \oplus \mathcal{M}^m
% $
% The final class probabilities after fusing evidence from all the modalities are obtained from the parameters of joint Dirichlet distribution \textbf{$\alpha$} which is given by: \\
% \begin{equation*}
%     S = \frac{K}{u} , e_k = b_k \times S, \alpha_k = e_k + 1
% \end{equation*}

For each of these fusion methods, we use the same backbone architecture to obtain the unimodal predictions. We train all models end to end via gradient descent and backpropagation to minimize the cross-entropy loss between the targets and predictions, using an Adam optimizer with a learning rate of $0.001$ and batch size of $128$.

\begin{table*}[!h]
\centering
\begin{tabular}{@{}llllll@{}}
\toprule
\multicolumn{1}{c}{\textbf{Fusion Model}}     & \multicolumn{1}{c}{Accuracy} & \multicolumn{1}{c}{Precision} & \multicolumn{1}{c}{Recall} & \multicolumn{1}{c}{F1Score} & \multicolumn{1}{c}{AUROC} \\ \midrule
MLP                                                        & $63.55 \pm 0.23$             & $ 64.65 \pm 2.24$             & $49.32 \pm 0.95$           & $52.35 \pm 0.68$            & $86.01 \pm 0.31$          \\
Weighted Mean (WM)                                             & $64.06 \pm 4.30$             & $ 64.70 \pm 1.38$             & $57.2  \pm 3.96$           & $59.17 \pm 3.22$            & $90.99 \pm 0.78$          \\
Noisy-OR                                                   & $66.71 \pm 1.42$             & $ 68.85 \pm 1.38$             & $59.06 \pm  1.21$ & $61.71 \pm 1.31$   & $91.23 \pm 0.31$          \\
TMC                                                        & $66.97 \pm 0.26$              & $ \mathbf{68.88 \pm 1.98}$            & $56.89 \pm 1.09$           & $59.94 \pm 0.42$            & $91.47 \pm 0.39$          \\
\midrule
Credibility-Weighted Mean (Ours)                                & $\mathbf{68.50 \pm 0.72}$     & $ 67.25 \pm 1.11$   & $\mathbf{60.17 \pm 0.85}$  & $\mathbf{62.03 \pm 0.91}$   & $\mathbf{91.52 \pm 0.41}$          \\ 
Direct-PC (Ours)                               & $57.64 \pm 2.01$                & $ 48.80 \pm 1.12$             & $49.84 \pm 1.46$           & $47.96 \pm 0.79$            & $79.70 \pm 0.62$          \\
\bottomrule
\end{tabular}
\caption{Mean test performance of late fusion methods on the \textbf{NYUD} dataset, $\pm$ standard deviation across $3$ trials.}\label{tab:results_nyud}
\end{table*}

\begin{table*}[!h]
\centering
\begin{tabular}{@{}llllll@{}}
\toprule
\multicolumn{1}{c}{\textbf{Fusion Model}}     & \multicolumn{1}{c}{Accuracy} & \multicolumn{1}{c}{Precision} & \multicolumn{1}{c}{Recall} & \multicolumn{1}{c}{F1Score} & \multicolumn{1}{c}{AUROC} \\ \midrule
MLP                                                        & $54.55 \pm 1.04$             & $ 46.40 \pm 0.15$             & $45.59 \pm 1.03$  &        $43.78 \pm 0.87$              & $87.19 \pm 0.38$          \\
Weighted Mean                                              & $51.80 \pm 2.29$             & $ 45.72 \pm 1.98$             & $42.94  \pm 0.73$          & $41.59 \pm 0.31$            & $90.21 \pm 0.78$          \\
Noisy-OR                                                   & $54.30 \pm 1.55$             & $ 46.76 \pm 1.34$             & $44.26 \pm  1.11$          & $43.60 \pm 0.95$            & $90.57 \pm 0.40$          \\
TMC                                                        & $50.92 \pm 1.66$             & $ 45.21 \pm 2.25$             & $42.94 \pm 0.57$           & $40.84 \pm 0.76$            & $89.84 \pm 0.32$          \\ 
\midrule
Credibility-Weighted Mean (Ours)                                & $\mathbf{57.97 \pm 1.05}$    & $ \mathbf{48.88 \pm 0.70}$    & $\mathbf{46.04 \pm 0.67}$  & $\mathbf{45.71 \pm 0.71}$   & $\mathbf{91.25 \pm 0.35}$          \\ 
Direct-PC (Ours)                               & $53.46 \pm 1.31$             & $ 41.97 \pm 0.68$             & $42.60 \pm 0.83$           & $40.73 \pm 0.76$            & $84.34 \pm 0.53$          \\ 
\bottomrule
\end{tabular}
\caption{Mean test performance of late fusion methods on the \textbf{SUNRGBD} dataset, $\pm$ standard deviation across $3$ trials.}\label{tab:results_sunrgbd}
\end{table*}

\textbf{Datasets.}
The CUB (\cite{WahCUB_200_2011}) dataset comprises of 11,788 images of birds, each annotated with attribute descriptions across 200 bird categories. Following \citeauthor{han2021trusted} (\citeyear{han2021trusted}), we used a subset of the original dataset consisting of the first 10 bird categories and 336 train images, 144 validation, and 120 test images for our experiments. Deep visual features obtained from using GoogLeNet on images, and the text features extracted using doc2vec are used as two modalities. 

The NYUD (\cite{Silberman2012IndoorSA}) is a widely used RGB-D scene recognition benchmark, containing RGB and Depth image pairs. Following previous work by \cite{zhang2023provable}, we use a reorganized dataset with 1863 image pairs (795 train, 414 validation, and 654 test) corresponding to 10 classes (9 usual scenes and one "others" category). The SUNRGBD (\cite{Song2015SUNRA}) is a relatively larger scene classification dataset with 10,335 RGB-depth image pairs. Following \cite{zhang2023provable}, we use a subset of the original dataset which contains the 19 major scene categories and 3876 train, 969 validation, and 4,659 test examples. In both the NYUD and SUNRGBD datasets, we utilized resnet18 ~\cite{He2015DeepRL} pre-trained on ImageNet as an encoder for each modality. 

AV-MNIST is a benchmark dataset designed for multimodal fusion. With 55,000 training, 5,000 validation, and 10,000 testing examples, it has two modalities: images of dimension 28 × 28 depicting digits from 0 to 9, and their corresponding audio represented as spectrograms of dimension 112 × 112. Following \cite{vielzeuf2018centralnet}, we used deep neural models with the LeNet architecture to encode the input data and make predictions for each modality. Specifically, we processed the image input through a 4-layer convolutional neural network with filter sizes [5, 3, 3, 3]. Similarly, the audio input was encoded using a 6-layer convolutional neural network with filter sizes [5, 3, 3, 3, 3, 3]. For all the datasets, the encodings obtained were processed through a feedforward neural network to obtain the unimodal predictions.

% The second data set, MM-IMDb, consists of movie poster images and textual data describing the movie plots. Given that each movie can be associated with multiple genres, the task here involves multilabel classification, in contrast to the multiclass classification task posed by AV-MNIST.

% \paragraph{Methods.}
% We compared our PC-based combining rule against 3 baseline combining rules -- weighted mean, noisy-or and Multilayer perceptron (MLP). 
% For the unimodal components (image and text classifiers) of MM-IMDb, we applied a two-layer perceptron encoder to each modality. The outputs of these encoders were then fed into a multi-layer perceptron (MLP) to generate the final predictions.

\subsection{Performance Evaluation}

\begin{figure}[!t]
    \centering
    \includegraphics[width=\linewidth]{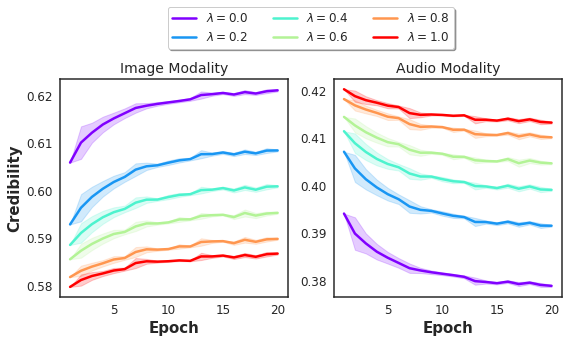}
    \caption{
    \textbf{Mean Validation Relative Credibility} obtained using a PC for the two modalities of the AV-MNIST dataset across training epochs. Varying degrees of noise (controlled by $\lambda$) are introduced into the audio modality. 
    % The shaded region represents the standard deviation across $3$ trials.
    }
    \label{fig:credibility-across-epoch}
    \vspace{-0.5cm}
\end{figure}

Table \ref{tab:results_avmnist} summarizes the test-set performance of the baseline methods and our PC-based combination functions on the AV-MNIST dataset in terms of the classification metrics - Accuracy, Precision, Recall, F1-Score, and AUC-ROC, after training for $50$ epochs with early stopping. We observe that our PC-based combination functions {\bf not only outperform simple probabilistic baselines such as Weighted Mean, Noisy-Or, and TMC on all performance metrics but also achieve performance similar to that of an MLP-based fusion} method.
Table \ref{tab:results_cub} summarizes the test-set performance of the baseline methods and our PC-based combination functions on the CUB dataset. We observe that our Credibility-Weighted Mean combination function achieves better performance than other models on average. As the CUB dataset is very small, we observed that complex models like MLP tend to overfit, impacting the test performance, while simpler combination functions like weighted mean and TMC achieved relatively better performance. Similar results obtained for the NYUD and SUNRGBD datasets are summarized in Tables \ref{tab:results_nyud} and \ref{tab:results_sunrgbd} respectively. Note that the NYUD dataset is also very small compared to the capacity of the resnet18-based models used as a backbone to encode the unimodal inputs. Here, again we can observe clearly that the relatively complex models like MLP and Direct-PC overfit, while simpler ones like Noisy-OR and Credibility-Weighted Mean generalize better.
% our Credibility Weighted Mean manages to achieve better performance due to the simple nature of the combination rule.
% Table \ref{tab:results_nyud} shows that our Credibility Weighted Mean performs the best out of all the models. 
% Note that the NYUD dataset is also very small for the model (resnet18) we are using as a backbone to encode the input from unimodals. Here, we can observe clearly that the relatively complex models overfit, effecting the performance. 
% Table \ref{tab:results_sunrgbd} summarizes the test-set performance on the SUNRGBD dataset. 
% We observe that our Credibility Weighted Mean outperforms all the models. 
Overall, the results suggest that the PC-based methods are expressive enough to capture intricate dependencies between unimodal predictive distributions and achieve performance at par and at times even better than more complex fusion approaches.

\subsection{Credibility Evaluation}

\begin{figure}[!h]
    \centering
    \includegraphics[width=\linewidth]{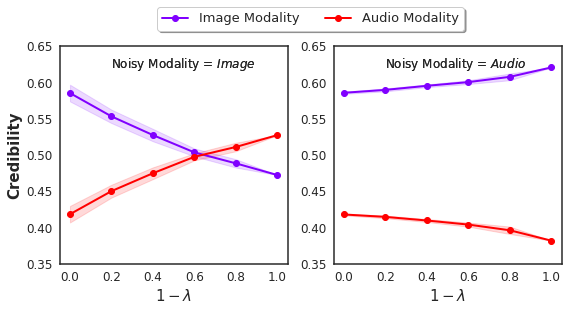}
    \caption{
    \small
    \textbf{Mean Test Relative Credibility} outputted by a PC for the two modalities of the AV-MNIST dataset across varying degrees of noise (controlled by $\lambda$) introduced into each modality. 
    % The shaded region represents the standard deviation across $3$ independent trials.
    }
    \label{fig:credibility-across-lamda}
\end{figure}
 
% Unlike, neural late fusion methods, the tractability offered by a PC for conditional and marginal inference allows us to define probabilistic measures for credibility of each modality, while making predictions, as defined in Eq. \ref{eq:credibility}. 
To empirically validate whether our PC-based late fusion method can reliably compute the credibility of each modality, we designed another experiment. We considered the AV-MNIST dataset and the Direct PC-based fusion model trained over it for $30$ epochs. We introduced varying degrees of noise into one of the modalities (say $i$), keeping the others fixed, and trained the PC to maximize the joint predictive likelihood. More specifically, we defined
\begin{equation*}
    \Tilde{P}(Y|\X_i) = \lambda P(Y|\X_i) + (1-\lambda) N
\end{equation*}
where $N \sim \text{Dir}(\alpha)$ is a noisy probability vector sampled from a Dirichlet distribution with parameters $\alpha$, and $0 \leq \lambda \leq 1$. $\Tilde{P}(Y|\X_i)$ is thus a convex combination of two probability distributions and is therefore a valid distribution. $\lambda$ controls the amount of information retained in  $\Tilde{P}$ from the unimodal predictive distribution. \\
% We use $\alpha=[0.5,0.5]$ in our experiments. \\
Note that as $\lambda \rightarrow 0$, $\Tilde{P}(Y|\X_i) \rightarrow N$, and thus has less predictive information about modality $i$. Thus, the credibility score should ideally decrease for modality $i$ and increase for the other modalities. Figure \ref{fig:credibility-across-epoch} shows how the mean relative credibility outputted by the PC over the validation set varies as it is trained over the noisy unimodal distributions with noise introduced into the audio modality, for varying values of $\lambda$. As expected, we can see that the credibility of the audio modality decreases as training progresses, while that of the image modality increases. Further, we can also observe that the decrease in credibility increases as $\lambda \rightarrow 0$. To demonstrate this correlation more evidently, we plot the Mean Relative Credibility outputted by the trained PC for each modality on the test set, for the two settings where noise is introduced into one of image/audio modalities in Figure \ref{fig:credibility-across-lamda}. We can clearly see that in both settings, the credibility score of the noisy modality decreases as $\lambda \rightarrow 0$, while that of the non-noisy modality increases. Thus, the credibility score outputted by the PC is a reliable measure that is reflective of the information contributed by each modality to the final predictive distribution. 

By averaging the credibility of each modality over all data points, we have so far looked at a \textit{global measure}, and the image modality seems to have higher global credibility than audio for AV-MNIST (see $\lambda=1$). However, the credibility of each modality may differ locally for individual data points, which can also be evaluated efficiently using the PC.

 \subsection{Robustness to Noise}
\begin{figure}[!h]
    \centering
    \includegraphics[width=0.49\linewidth]{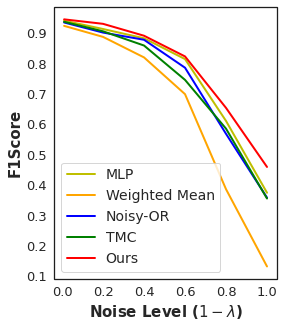}\includegraphics[width=0.5\linewidth]{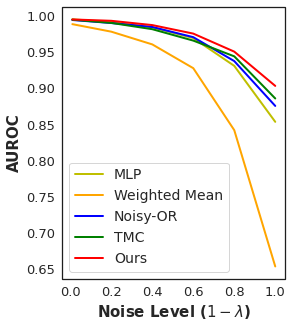}
    \caption{
    \small
    \textbf{Robustness to Noise.} Mean test performance of late fusion methods across varying degrees of noise.} 
    \label{fig:robustness}
\end{figure}
We also evaluated the robustness of our proposed credibility-aware late fusion methodology to noisy unimodal predictive distributions. Figure \ref{fig:robustness} depicts the decline in test performance for the different fusion methods over the CUB dataset when varying degrees of noise $\lambda$ are introduced in one of the unimodal predictive distributions. We can observe that our approach suffers the lowest decline in terms of both F1 score and AUROC, validating the robustness of our approach.

% \paragraph{Results.}
% \begin{enumerate}[leftmargin=2.5em]
%     \item[\textbf{(Q1)}] Table \ref{tab:results} presents the test-set performance of the baseline models and our PC model on the AV-MNIST data set. The PC based model outperforms weighted mean and noisy-or based models on all metrics, and has identical performance to the MLP based model. 
%     \item[\textbf{(Q2)}] ---
% \end{enumerate}
% \begin{figure}
%     \centering
%     \begin{tabular}{ccc}
%     \subfloat[a]{\includegraphics[width=0.5\linewidth]{figures/noisy_modality-1_Val_Credibility-Modality-1.png}}
%     \subfloat[a]{\includegraphics[width=0.5\linewidth]{figures/noisy_modality-1_Val_Credibility-Modality-0.png}}
%     \end{tabular}
%     \caption{}
%     \label{fig:credibility}
% \end{figure}

% \section{To Do}
% \begin{enumerate}
%     \item Train on non-noisy test on noisy distribution
%     \item Train on noisy test on noisy distribution
%     \item Weighted Mean Combination function, where weights are \textbf{credibility} outputted by PC. If we use this, we can also make use of the Rademacher complexity results.
% \end{enumerate}

\section{Conclusion}
We considered the problem of late multi-modal fusion in the noisy discriminative learning setting. We derived a theoretically grounded measure of credibility and proposed probabilistic circuit-based combination functions for late-fusion that are expressive enough to model complex interactions, robust to missing modalities, and capable of making reliable and credibility-aware predictions. Our experiments demonstrated that the proposed approach is competitive with the state-of-the-art while allowing for a principled way to infer the credibility of each modality. Scaling the approach to domains with more sources and extending the framework to allow subgroup-specific credibilities are promising directions for future research.

% \begin{contributions} % will be removed in pdf for initial submission 
% 					  % (without ‘accepted’ option in \documentclass)
%                       % so you can already fill it to test with the
%                       % ‘accepted’ class option
%     Briefly list author contributions. 
%     This is a nice way of making clear who did what and to give proper credit.
%     This section is optional.

%     H.~Q.~Bovik conceived the idea and wrote the paper.
%     Coauthor One created the code.
%     Coauthor Two created the figures.
% \end{contributions}

\begin{acknowledgements} 
The authors (\textit{SS, PT, SM} and \textit{SN}) gratefully acknowledge the generous support by the AFOSR award FA9550-23-1-0239, the ARO award W911NF2010224 and the DARPA Assured Neuro Symbolic Learning and Reasoning (ANSR)
award HR001122S0039. The views and conclusions contained herein are those
of the authors and should not be interpreted as necessarily representing the official policies or endorsements, either expressed or implied, by AFOSR, ARO, DARPA or
the US government.
\end{acknowledgements}

% References
\bibliography{references}

% \clearpage
\onecolumn
% \title{Credibility-Aware  Multi-Modal Fusion Using Probabilistic Circuits\\(Supplementary Material)}

% \maketitle

\section*{Appendix}
\appendix
\section{Theorems and Proofs}
\begin{theorem}
     The expected credibility $\mathcal{C}^{j}$ of a modality $j$ in predicting the target $Y,$ under a \pname{} distribution, is lower bounded by the negative of the conditional entropy $(\mathbb{H})$ of the unimodal predictive distribution of modality $j$ over $Y,$ given the predictive distributions of all other modalities, i.e. 
    \begin{equation*}
    \mathbb{E}[\mathcal{C}^j] \geq - \mathbb{H}(\mathcal{F}_{\phi_j} | \{\mathcal{F}_{\phi_i}\}_{i=1}^{M} \setminus \{\mathcal{F}_{\phi_j}\})
    \end{equation*}
\end{theorem}
\begin{proof}
    For ease, let us use the notation $\F = \{\mathcal{F}_{\phi_i}\}_{i=1}^{M} $ and $\F^{-j} = \{\mathcal{F}_{\phi_i}\}_{i=1}^{M} \setminus \{\mathcal{F}_{\phi_j}\}$. We have from the definition of credibility, using KL divergence as the divergence measure,
    \begin{align*}
        \mathcal{C}^j ={}&  \delta(P(Y|\F)||P(Y|\F^{-j}))\\
        ={}& \sum_y P(y|\F) \log \frac{P(y|\F)}{P(y|\F^{-j})} \\ 
        ={}&  \frac{1}{P(\F)} \sum_y P(y,\F) \log \frac{P(y,\F)P(\F^{-j})}{P(\F)P(y,\F^{-j})} \\
        ={}& \frac{1}{P(\F)} \sum_y P(y,\F) \log \frac{P(y,\F)}{P(y,\F^{-j})} + \log \frac{P(\F^{-j})}{P(\F)}
    \end{align*}
Now, we know that $\log \frac{P(\F^{-j})}{P(\F)} \geq 0$ as $P$ is assumed to be \pname{}. 
% and $P(x_1, x_2) \geq 0, \quad \forall x_1, x_2 \implies \sum_{x_2} P(x_1, x_2) \geq P(x_1, x_2)$
Thus, we have
\begin{align*}
    \mathcal{C}^j \geq{}& \frac{1}{P(\F)} \sum_y P(y,\F) \log \frac{P(y,\F)}{P(y,\F^{-j})}
\end{align*}
Now, applying the log sum inequality $\sum_i a_i \log \frac{a_i}{b_i} \geq \Bar{a}\log\frac{\Bar{a}}{\Bar{b}}$ where, $\Bar{a}=\sum_i a_i, \Bar{b}=\sum_i b_i$, and taking expectations, we get
\begin{align*}
    \mathbb{E}[\mathcal{C}^j] \geq{}& \mathbb{E}[\frac{1}{P(\F)} (\sum_y P(y,\F)) \log \frac{\sum_y P(y,\F)}{\sum_y P(y,\F^{-j})}] \\
     ={}& \mathbb{E}[\log \frac{P(\F)}{P(\F^{-j})}] = \mathbb{E}[\log \frac{P(\F^{-j}, \mathcal{F}_{\phi_j})}{P(\F^{-j})}]
\end{align*}
Using the definition of conditional entropy $\mathbb{H}(Y|X) = \mathbb{E}[-\log \frac{P(X,Y)}{P(X)}]$, the above inequality reduces to 
\begin{align*}
    \mathcal{C}^j \geq{}& - \mathbb{H}(\mathcal{F}_{\phi_j}|\F^{-j})
\end{align*}
\end{proof}

\begin{theorem}
    A Probabilistic Circuit is \pname \ if it is smooth, decomposable and has leaf distributions with unimodal densities upper-bounded by unity.
\end{theorem}
\begin{proof}
    Consider a PC $\mathcal{M}$ representing the distribution over $n$ variables $\X.$ Without loss of generality let $j$ denote the index of the variable being marginalized. Recall that $\mathcal{M}$ is said to be \pname \ if $P_\mathcal{M}(\X^{-j} = \x^{-j}) \geq P_\mathcal{M}(\X^{-j} = \x^{-j}, \X^{j} = \x^{j})\ \  \forall \ (\x^j,\x^{-j}) \in Dom.(\X^j, \X^{-j})$.
    
    As PCs are recursively defined as compositions of three types of nodes - sums, products and univariate leaf distributions in the form of a rooted directed acyclic graph, we can prove by induction on the height of a PC that the introduction of each type of node preserves marginal dominance under the structural properties of smoothness, decomposability and unity bounded leaf densitites.
    
    As the base case, consider any univariate leaf node $l$ in $\mathcal{M}.$ We have
    \begin{equation*}
        P_l(\X = \x) = \psi_l(\X^{\textbf{sc}(l)} = \x^{\textbf{sc}(l)})
    \end{equation*}
    where $\psi_l$ denotes the leaf density function and $\textbf{sc}(l)$ denotes the scope of $l$. Now, if ${\textbf{sc}(l)} = j$, then
    \begin{equation*}
        P_l(\X^{-j} = \x^{-j}) = \int_{\x^j} P_l(\X^{-j} = \x^{-j}, \X^j = \x^j) d\x^j = 1 \geq P_l(\X^{-j} = \x^{-j}, \X^j = \x^j)
    \end{equation*}
    since the leaf densities are upper bounded by unity. On the other hand, if ${\textbf{sc}(l)} \ne j$, then $P_l(\X^{-j})=P_l(\X^{-j}, \X^j)$ trivially. Thus under both cases, leaf nodes are \pname, hence the base case is satisfied.
    
    Now, let us assume that all nodes at height $K-1$ in the PC satisfies marginal dominance. We will show that all nodes at height $K$ also satisfies marginal dominance. Note that as sums and products constitute the internal nodes in a PC any node at height $K$ is obtained by introducing either a sum node or a product node over nodes at height $K-1$. Let us consider the two cases separately.

    Let $\times \in \mathcal{M}$ denote a decomposable product node at height $K$. We have 
    \begin{equation*}
        P_\times(\X = \x) = \prod_{c\in \textbf{ch}(\times)} P_c(\X^{\textbf{sc}(c)} = \x^{\textbf{sc}(c)})
    \end{equation*}
    where $\textbf{ch}(\times)$ denotes the children of $\times$.
    Thus, if $\times$ is decomposable then $\X^{j}$ can be present in the scope of only one of its children, say $\mathcal{N}$. Thus we have, 
    \begin{equation*}
        P_\times(\X^{-j} = \x^{-j}, \X^j=\x^j) = \left[ \prod_{c\in \textbf{ch}(\times), j \not\in \textbf{sc}(c)} P_c(\X^{\textbf{sc}(c)} = \x^{\textbf{sc}(c)}) \right] P_{\mathcal{N}}(\X^{\textbf{sc}(\mathcal{N}) \setminus j}=\x^{\textbf{sc}(\mathcal{N}) \setminus j}, \X^j=\x^j)
    \end{equation*}
    Now, since $\mathcal{N}$ is a node of height atmost $K-1$, by the inductive assumption, it is \pname. Hence, $\forall \x^j \in Dom.(\X^j)$, we have
    \begin{align*}
    P_\times(\X^{-j} = \x^{-j}) &=  \left[ \prod_{c\in \textbf{ch}(\times), j \not\in \textbf{sc}(c)} P_c(\X^{\textbf{sc}(c)} = \x^{\textbf{sc}(c)}) \right] P_{\mathcal{N}}(\X^{\textbf{sc}(\mathcal{N}) \setminus j}=\x^{\textbf{sc}(\mathcal{N}) \setminus j})  \\
    &\geq \left[ \prod_{c\in \textbf{ch}(\times), j \not\in \textbf{sc}(c)} P_c(\X^{\textbf{sc}(c)} = \x^{\textbf{sc}(c)}) \right] P_{\mathcal{N}}(\X^{\textbf{sc}(\mathcal{N}) \setminus j}=\x^{\textbf{sc}(\mathcal{N}) \setminus j}, \X^j=\x^j)\\
    &= P_\times(\X^{-j} = \x^{-j}, \X^j=\x^j) 
    \end{align*}

    Thus, since the product of non-negative terms preserves the direction of the inequality and $\mathcal{N}$ is \pname \ , the product node $\times$ is also \pname.

    Now, let $+ \in \mathcal{M}$ denote a smooth sum node at height $K$. We have  
    \begin{equation*}
        P_+(\X^{-j} = \x^{-j}, \X^j=\x^j) = \sum_{c\in \textbf{ch}(+)} w_c P_c(\X^{-j} = \x^{-j}, \X^j=\x^j)
    \end{equation*}
    where $0\leq w_c \leq 1 \ \forall w_c$ and $\sum_{c \in \textbf{ch}(+)} w_c = 1$. Since each $c \in \textbf{+}$ is a PC node of height atmost $K-1$, it is marginal dominant by the inductive assumption. Thus we have, $\forall \x^j \in Dom.(\X^j)$,

    \begin{align*}
        P_+(\X^{-j} = \x^{-j}) &= \sum_{c\in \textbf{ch}(+)} w_c P_c(\X^{-j} = \x^{-j}) \\
        &\geq  \sum_{c\in \textbf{ch}(+)} w_c P_c(\X^{-j} = \x^{-j}, \X^j = \x^j)\\
        &= P_+(\X^{-j} = \x^{-j}, \X^j=\x^j)
    \end{align*}
    
    \textit{i.e.,} $+$ is  \pname \ which follows from the fact that the convex combination preserves the direction of the inequality. 

    Thus all nodes at height $K$ are also marginal dominant, and by principle of mathematical induction, we can conclude that the PC is \pname.

\end{proof}

% \begin{theorem}
%     A node $n$ in a PC where the leaf densities are upper-bounded by 1 is \pname{} if for any $j \in \textbf{sc}(n),$
%     \begin{equation*}
%         P_n(\X^{-j} = \x^{-j}) \leq P_n(\X^{-j} = \x^{-j}, \X^{j} = \x^{j})\ \forall x^j \in Dom.(X^j)
%     \end{equation*}
% \end{theorem}
% \begin{proof}
%     Consider a PC $\mathcal{M}$ representing the distribution over $n$ variables $\X.$ Without loss of generality let $j$ be the variable being marginalized. 

%     Consider a leaf node $n$ in $\mathcal{M}.$ 
%     \begin{equation*}
%         P_n(\X = \x) = \psi_n(\X^{\textbf{sc}(n)} = \x^{\textbf{sc}(n)})
%     \end{equation*}
%     Here, if ${\textbf{sc}(n)} = j$, then $P_n(\X = \x) = 1.$ So, leaf nodes satisfy \pname.
    
%     Consider a product node $n$ in $\mathcal{M}.$ 
%     \begin{equation*}
%         P_n(\X = \x) = \prod_{c\in \textbf{ch}(n)} P_c(\X^{\textbf{sc}(c)} = \x^{\textbf{sc}(c)})
%     \end{equation*}
%     If each child node $c$ satisfies \pname, then product nodes satisfy \pname since the product of non-negative terms preserves the direction of the inequality.

%     Consider a sum node $n$ in $\mathcal{M}.$ 
%     \begin{equation*}
%         P_n(\X = \x) = \sum_{c\in \textbf{ch}(n)} w_c P_c(\X = \x)
%     \end{equation*}
%     If each child node $c$ satisfies \pname, then sum nodes satisfy \pname since the convex combination of non-negative terms preserves the direction of the inequality.

% \end{proof}

\section{Implementation Details}
\textbf{Datasets.}
The CUB (\cite{WahCUB_200_2011}) dataset comprises of 11,788 images of birds, each annotated with attribute descriptions across 200 bird categories. Following \citeauthor{han2021trusted} (\citeyear{han2021trusted}), we used a subset of the original dataset consisting of the first 10 bird categories and 336 train images, 144 validation, and 120 test images for our experiments. Deep visual features obtained from using GoogLeNet on images, and the text features extracted using doc2vec are used as two modalities. 

The NYUD (\cite{Silberman2012IndoorSA}) is a widely used RGB-D scene recognition benchmark, containing RGB and Depth image pairs. Following previous work by \cite{zhang2023provable}, we use a reorganized dataset with 1863 image pairs (795 train, 414 validation, and 654 test) corresponding to 10 classes (9 usual scenes and one "others" category). The SUNRGBD (\cite{Song2015SUNRA}) is a relatively larger scene classification dataset with 10,335 RGB-depth image pairs. Following \cite{zhang2023provable}, we use a subset of the original dataset which contains the 19 major scene categories and 3876 train, 969 validation, and 4,659 test examples. In both the NYUD and SUNRGBD datasets, we utilized resnet18 ~\cite{He2015DeepRL} pre-trained on ImageNet as an encoder for each modality. 

AV-MNIST is a benchmark dataset designed for multimodal fusion. With 55,000 training, 5,000 validation, and 10,000 testing examples, it has two modalities: images of dimension 28 × 28 depicting digits from 0 to 9, and their corresponding audio represented as spectrograms of dimension 112 × 112. Following \cite{vielzeuf2018centralnet}, we used deep neural models with the LeNet architecture to encode the input data and make predictions for each modality. Specifically, we processed the image input through a 4-layer convolutional neural network with filter sizes [5, 3, 3, 3]. Similarly, the audio input was encoded using a 6-layer convolutional neural network with filter sizes [5, 3, 3, 3, 3, 3]. For all the datasets, the encodings obtained were processed through a feedforward neural network to obtain the unimodal predictions.

\section{Experimental Setup}
For the experiments, we utilized Intel Xeon Platinum 8167M CPU with 24 cores along with NVIDIA Tesla V100 GPUs, each with 16GB memory. Our setup included a total of 2 GPUs, enabling us to distribute the workload efficiently across CUDA cores. However, our experimental results can be reproduced using a single GPU instance of the V100 with the aforementioned configuration.

A total of 8 workers were used to load, preprocess and train the model for each of the datasets. The compute time for the experiment when run on a single GPU instance was approximately an hour for each configuration of the combination functions for the NYUD and AV-MNIST datasets whereas it took only 6 minutes for CUB dataset due to it's compact size. SUN-RGBD, on the other hand, took about 5 hours to run each configuration as it's huge in size, compared to other datasets. Memory utilization was closely monitored, and we observed an approximate average usage of 1, 9, 2 and 9 GB for CUB, NYUD, AVMNIST and SUNRGBD respectively.

% \newpage

% \title{Title in Title Case\\(Supplementary Material)}
% \maketitle

% This Supplementary Material should be submitted together with the main paper.

% \appendix
% \section{Additional simulation results}
% Table~\ref{tab:supp-data} lists additional simulation results; see also \citet{einstein} for a comparison. 

% \begin{table}[!h]
%     \centering
%     \caption{An Interesting Table.} \label{tab:supp-data}
%     \begin{tabular}{rl}
%         \toprule % from booktabs package
%         \bfseries Dataset & \bfseries Result\\
%         \midrule % from booktabs package
%         Data1 & 0.12345\\
%         Data2 & 0.67890\\
%         Data3 & 0.54321\\
%         Data4 & 0.09876\\
%         \bottomrule % from booktabs package
%     \end{tabular}
% \end{table}

% \section{Math font exposition}
% % NOTE: necessary when ptmx or no mathfont class option is given
% \providecommand{\upGamma}{\Gamma}
% \providecommand{\uppi}{\pi}
% How math looks in equations is important:
% \begin{equation*}
%     F_{\alpha,\beta}^\eta(z) = \upGamma(\tfrac{3}{2}) \prod_{\ell=1}^\infty\eta \frac{z^\ell}{\ell} + \frac{1}{2\uppi}\int_{-\infty}^z\alpha \sum_{k=1}^\infty x^{\beta k}\mathrm{d}x.
% \end{equation*}
% However, one should not ignore how well math mixes with text:
% The frobble function \(f\) transforms zabbies \(z\) into yannies \(y\).
% It is a polynomial \(f(z)=\alpha z + \beta z^2\), where \(-n<\alpha<\beta/n\leq\gamma\), with \(\gamma\) a positive real number.

\end{document}